\documentclass[runningheads]{llncs}

\usepackage[utf8]{inputenc}
\usepackage[T1]{fontenc}
\usepackage{microtype}
\usepackage{fullpage}
\usepackage{amsmath}
\usepackage{amssymb}
\usepackage{mathtools}
\usepackage[inline]{enumitem}
\usepackage[hyphens]{url}
\usepackage{tikz}
\usepackage{todonotes} 
\usepackage[hidelinks]{hyperref} 
\usepackage{booktabs}

\usetikzlibrary{automata, arrows}
\tikzset{every edge/.append style={shorten >= 1pt}}
\tikzset{auto, >= stealth}
\usepackage{tikz}
\usetikzlibrary{bayesnet,arrows}
\usetikzlibrary{shapes,fit,calc,positioning}
\tikzset{box/.style={draw, diamond, thick, text centered, minimum height=0.5cm, minimum width=1cm, text width=0.9cm}}
\tikzset{line/.style={draw, thick, -latex'}}
\allowdisplaybreaks

\newcommand{\framework}{\ensuremath{\mathsf{pacX}}}

\title{Probably Approximately Correct Explanations of Machine Learning Models via Syntax-Guided Synthesis}

\author{
	Daniel Neider\inst{1} \and
	Bishwamittra Ghosh\inst{2}
}

\institute{
	Max Planck Institute for Software Systems, Kaiserslautern, Germany (\email{neider@mpi-sws.org}) \and
	National University of Singapore, Singapore (\email{bghosh@u.nus.edu})
}

\begin{document}

\maketitle

\begin{abstract}
We propose a novel approach to understanding the decision making of complex machine learning models (e.g., deep neural networks) using a combination of probably approximately correct learning (PAC) and a logic inference methodology called syntax-guided synthesis (SyGuS).
We prove that our framework produces explanations that with a high probability make only few errors and show empirically that it is effective in generating small, human-interpretable explanations.

\end{abstract}


\section{Introduction}
\label{sec:intro}

Recent advances in artificial intelligence and machine learning, especially in deep neural networks, have shown the great potential of algorithmic decision making in a host of applications.
The inherent black-box nature of today’s complex machine learning models, however, has raised concerns regarding their safety, reliability, and fairness.
In fact, the lack of explanations as to why a learning-based system has made a certain decision has not only been identified as a major problem by the scientific community but also by society at large.
One such example is the European Union, who considers imposing a ``right to explanation'' to future algorithmic decision making~\cite{goodman2017european}.

As a step towards the explanation of algorithmic decision making, this paper proposes a novel framework to generate human-interpretable descriptions (explanations) of the decision boundary of black-box machine learning models.
The defining features of our framework are fourfold:
\begin{enumerate}
	\item In contrast to current efforts in the literature~\cite{DBLP:conf/aaai/IgnatievNM19,DBLP:conf/ijcai/Ignatiev20}, we view machine learning models as black boxes.
	This choice is motivated by two observations.
	Firstly, despite the immense advances in verifying deep neural networks~\cite{DBLP:conf/atva/Ehlers17,DBLP:conf/cav/KatzBDJK17,DBLP:conf/cav/ElboherGK20}, reasoning about realistic machine learning models in a symbolic, white-box manner is currently not computationally tractable---deep neural networks, as deployed in real-world applications, are simply too large.
	Secondly, one might not have direct access to the machine learning model but can only observe its input-output behavior.
	%
	\item We use quantifier-free formulas in first-order logic (FO) as explanations.
	Due to the declarative nature of FO, we believe that such explanations are generally easy for humans to understand and analyze---at least up to a certain ``size'' of the formula.
	In fact, FO has already been used successfully to explain deep neural networks~\cite{DBLP:conf/aaai/IgnatievNM19,DBLP:conf/ijcai/Ignatiev20}.
	\item We are not interested in a global explanation because a description of the decision boundary in the whole input-space (or a large part of it) will be too complex.
	Instead, we want to generate an explanation within a user-defined region of the input-space, which we call \emph{query}.
	A query can either be an area of interest in the input space or a ``ball'' around some user-given input.
	Note, however, that this is not a restriction: a query satisfying all inputs amounts to a global explanation as it covers to the whole input-space.
	\item Since an exact descriptions of the decision boundary (even inside a query) can be very complex and too difficult to understand, we compute approximate explanations that, with a high confidence, make only few errors.
	Inspired by Valiant's probably approximately correct (PAC) learning framework~\cite{DBLP:journals/cacm/Valiant84}, we term this the \emph{PAC property} and call our explanations \emph{PAC explanations}.
	More precisely, given two parameters $\varepsilon, \delta \in (0, 1)$, our framework guarantees (on termination) to generate an explanation where, with confidence at least $1 - \delta$, the probability of making an error in explaining the decision boundary is at most $\varepsilon$.
\end{enumerate}

As shown in Figure~\ref{fig:learning-epsilon-explanations}, our framework follows the principle of \emph{iterative passive learning}~\cite{DBLP:journals/tc/BiermannF72} (often referred to as \emph{counterexample-guided inductive synthesis}~\cite{CEGIS}) and consists of a feedback loop with two components: a synthesizer and a verifier.

The task of the synthesizer is to generate a hypothesis explanation (in form of a quantifier-free FO formula) from concrete input-output behavior of the machine learning model.
To this end, we propose the use of a general framework called \emph{Syntax-Guided Synthesis (SyGuS)}~\cite{DBLP:conf/fmcad/AlurBJMRSSSTU13}, which has recently gained significant attention in the context of program synthesis.
This framework offers an effective and simple way to synthesize quantifier-free FO formulas from both semantic and syntactic constraints.
In our setting, the semantic constraints correspond to input-output examples of the machine learning model, which the synthesizer obtains from the verifier (as described shortly).
The syntactic constraints allow restricting the space of possible explanations (i.e., FO formulas).
This is useful to introduce domain-knowledge (e.g., if one is only interested in particular explanations or knows the general ``pattern'' of explanations) and to reduce the computational effort required for synthesizing an explanation.

The task of the verifier, on the other hand, is to check whether the synthesizer's hypothesis is a sufficiently accurate explanation using a statistical test inspired by Valiant's PAC learning~\cite{DBLP:journals/cacm/Valiant84}.
If the hypothesis does not pass this test, the verifier returns input-output samples from the test suite that witness an error of the explanation (i.e., that fail the test).
Such  samples, which are called \emph{counterexamples}, refute the current hypothesis and guide the synthesizer towards a more accurate explanation.
The process of conjecturing explanations and checking them repeats until a sufficiently accurate explanation is found (i.e., one that passes the statistical test and, hence, satisfies the PAC property).
Once this happens, the feedback loop terminates and returns this hypothesis.

We prove that our framework guarantees to output a PAC explanation if it terminates.
However, the decision boundary of the model might be too complex, even in a PAC sense, causing our framework to loop forever.
To alleviate this problem, we identify sufficient conditions for which our framework is guaranteed to terminate.
If these conditions are met, our framework either returns a PAC explanation or reports that no FO explanation exists.

In our empirical evaluations, we show the correctness of the generated explanations on a decision tree classifier as a black-box. Furthermore, we  extend experiments to practical machine learning datasets where {\framework} generates small-size interpretable explanations of the neural networks.


\section{Preliminaries}
\label{sec:preliminaries}

In this paper, we adopt a black-box view on machine learning models.
More precisely, we view a machine learning model, \emph{model} for short, as a function $\mathcal M \colon \mathbb R^n \to C$ where $n \in \mathbb N \setminus \{ 0 \}$ is the input dimension of the model and $C = \{ c_1, \ldots, c_m \}$ is a finite set of classes.
As an illustrative example, we encourage the reader to think of multi-class deep neural networks.

A \emph{query} is a subset $Q \subseteq \mathbb R^n$.
Usually, we use polytopes (i.e., intersections of half-spaces of the $\mathbb R^n$) to represent queries because polytopes have nice algorithmic properties.
However, other symbolic representations of subsets of the $\mathbb R^n$ are also possible.
We encourage the reader to think of a query as a box.

We denote first order formulas by small Greek symbols $\varphi$, $\psi$, and so on.
For instance, the FO formula $\varphi \coloneqq (-1 \leq x_1) \land (x_1 \leq 1) \land (-1 \leq x_2) \land (x_2 \leq 1)$ defines a unit box around the origin in the $\mathbb R^2$.
Given an input $\vec{x} \in \mathbb R^n$, we define satisfaction of FO formulas using a relation $\models$ in the usual way; in particular, we write $\vec{x} \models \varphi$ if $\vec{x}$ \emph{satisfies} $\varphi$.

Given a model $\mathcal M \colon \mathbb R^n \to C$, a class $c \in C$, and a query $\psi$, a \emph{(perfect) explanation} is a quantifier-free FO formula $\varphi$ such that
\begin{align}
	\label{eq:explanation}
	\text{if } \vec{x} \models \psi \text{, then } \vec{x} \models \varphi \text{ if and only if } \mathcal M(\vec{x}) = c.
\end{align}
In other words, the explanation should exactly describe the decision boundary for label $c$ inside the query $\psi$.

By definition, $\varphi$ is not an explanation if there exists an $\vec{x} \in \mathbb R^n$ that violates Eq.~\eqref{eq:explanation}.
Such an input then either satisfies
\begin{enumerate*}[label={(\alph*)}]
	\item $\vec{x} \models \psi$, $\vec{x} \models \varphi$, and $\mathcal M(\vec{x}) \neq c$ or
	\item $\vec{x} \models \psi$, $\vec{x} \not\models \varphi$, and $\mathcal M(\vec{x}) = c$.
\end{enumerate*}
(i.e., the input lies inside the query and shows a difference between $\varphi$ and $\mathcal M$).
For a given FO formula $\varphi$, a query $\psi$, and a class $c \in C$, let
\[ V_\mathcal M(\varphi, \psi, c) \coloneqq \bigl\{ \vec{x} \in \mathbb R^n \mid \text{$\vec{x}$ violates Eq.~\eqref{eq:explanation}} \bigr\} \]
denote the set of all inputs $\vec{x}$ that do not satisfy Eq.~\eqref{eq:explanation}.

In this paper, we are not interested in exact explanations (as they will be to complex) but in approximate ones.
To make this idea precise, let us fix a probability distribution $\mathcal D$ over the input-space $\mathbb R^n$.
Moreover, let $\varepsilon \in (0, 1)$.
Then, we call a quantifier-free FO formula $\varphi$ an \emph{$\varepsilon$-explanation} if 
\[ \mathbb P_\mathcal D \bigl[ \vec{x} \in V_\mathcal M(\varphi, \psi, c) \bigr] < \varepsilon \]
(i.e., the probability of a randomly chosen input $\vec{x}$ revealing that $\varphi$ is not an explanation is less than $\varepsilon$).
In other words, the probability of $\varphi$ making an error in describing the decision boundary is less than $\varepsilon$.
This allows us to generate ``simpler'' explanations (as compared to perfect ones), which make small errors but are human-interpretable.

In the remainder, we solve the following problem.

\begin{problem} \label{prob:computing-epsilon-explanations}
Let $\mathcal M \colon \mathbb R^n \to C$ be a model with $C = \{ c_1, \ldots, c_m \}$, $c \in C$ a class, and $\psi$ a query.
Moreover, let $\varepsilon, \delta \in (0, 1)$.
Design an algorithm that generates an $\varepsilon$-explanation of $\mathcal M$ inside $\psi$ with probability at least $1 -\delta$.
\end{problem}

Since the algorithm we design to solve Problem~\ref{prob:computing-epsilon-explanations} is independent of the probability distribution $\mathcal D$, we drop the subscript whenever it is clear from the context.


\section{Generating PAC Explanations}
\label{sec:PAC-explanations}

Our PAC-learning framework for synthesizing $\varepsilon$-explanations, named \framework, is shown in Figure~\ref{fig:learning-epsilon-explanations}.
It consists of a feedback loop with two entities: a \emph{synthesizer} and a \emph{verifier}.
In every iteration of the feedback loop, the synthesizer conjectures a quantifier-free FO formula $\varphi$ based on the information it has gathered so far.
The verifier, on the other hand, checks whether the proposed formula $\varphi$ is an explanation by means of randomly sampling the model.
If the verifier detects an input $\vec{x} \in V_\mathcal M(\varphi, \psi, c)$ (i.e., an input that violates Eq.~\eqref{eq:explanation} and, hence, witnesses that the conjectured explanation $\varphi$ is incorrect on the input $\vec{x}$), it returns a so-called \emph{counterexample} $(\vec{x}, \ell)$ where $\ell \in \{ 0, 1 \}$ is the expected behavior of any future conjecture on the input $\vec{x}$ (we explain shortly how this can be derived).
After receiving a counterexample, the synthesizer refines its conjecture and proceeds to the next round of the feedback loop.
This process continues until the verifier cannot find any counterexample to the current conjecture.
Then, the feedback loop terminates and returns the most recent conjecture.
As we show in Section~\ref{sec:theoretical-analysis}, this conjecture is in fact an $\varepsilon$-explanation with probability at least $1 - \delta$.

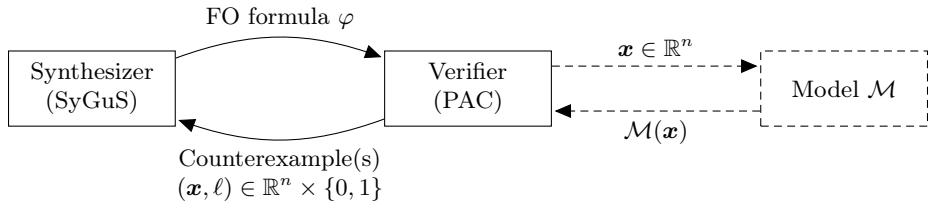
\begin{figure*}[t!h]
	\centering
	\begin{tikzpicture}
		\begin{scope}[every node/.append style={draw, text width=20mm, minimum height=10mm, align=center}]
			\node (S) at (0, 0) {Synthesizer \\ (SyGuS)};
			\node (V) at (5, 0) {Verifier \\ (PAC)};
			\node[densely dashed] (DNN) at (10, 0) {Model $\mathcal M$};
		\end{scope}

		\path[->] (S) edge[bend left=20] node {FO formula $\varphi$} (V);
		\path[->] (V) edge[bend left=20] node[align=center] {Counterexample(s) \\ $(\vec{x}, \ell) \in \mathbb R^n \times \{ 0, 1 \}$} (S);
		\path[->, densely dashed] (V.15) edge node {$\vec{x} \in \mathbb R^n$} (DNN.165);
		\path[->, densely dashed] (DNN.195) edge node {$\mathcal M(\vec{x})$} (V.-15);
	\end{tikzpicture}
	\caption{\framework: a PAC-learning framework for synthesizing $\varepsilon$-explanations} \label{fig:learning-epsilon-explanations}
\end{figure*}

In the remainder of this section, we describe  the verifier and the synthesizer in detail.

\subsection{Verifier}
\label{sec:verifier}
Given an FO formula $\varphi$, the task of the verifier is to check whether it is an $\varepsilon$-explanation.
To this end, the verifier generates a finite test suite $T_i \subset \mathbb R^n$ of inputs to the model that are randomly drawn according to the given probability distribution $\mathcal D$.
The size of $T_i$ depends on $\varepsilon$, $\delta$, and the current iteration $i \geq 1$ of the feedback loop and is given by
\[ |T_i| = \left\lceil \frac{1}{\varepsilon} (i \cdot \ln{2} - \ln{\delta}) \right\rceil. \]

For each test input $\vec{x} \in T_i$, the verifier checks whether this input satisfies Eq.~\eqref{eq:explanation} (i.e., $\vec{x} \notin V_\mathcal M(\varphi, \psi, c)$?).
This involves three checks:
\begin{enumerate}
	\item \label{itm:check-1} checking whether $\vec{x} \models \varphi$ holds;
	\item \label{itm:check-2} checking whether $\vec{x} \models \psi$ holds; and
	\item \label{itm:check-3} checking whether $\mathcal M(\vec{x}) = c$ holds.
\end{enumerate}
Check~\ref{itm:check-3} can be done in a straightforward manner by passing the input $\vec{x}$ through the model and checking whether the output is class $c$.
Checks~\ref{itm:check-1} and \ref{itm:check-2} are equally straightforward and involve the evaluation of the FO formulas $\varphi$ and $\psi$ on the input $\vec{x}$.
Since both $\varphi$ and $\psi$ are quantifier-free, this can be done by a simple procedure which recursively evaluates the formulas along their syntactic structure.
Note that all three checks can be efficiently executed in parallel, and multiple test inputs can be checked at the same time.

If none of the inputs in the test suite $T_i$ violate Eq.~\eqref{eq:explanation}, the verifier stops the feedback loop and returns the current formula $\varphi$---we show shortly that $\varphi$ is then an $\varepsilon$-explanation with probability at least $1 - \delta$.
However, if there exists an input $\vec{x}$ violating Eq.~\eqref{eq:explanation}, then the verifier returns a pair $(\vec{x}, \ell)$ as a counterexample.
Since $\vec{x}$ witnesses a violation of Eq.~\eqref{eq:explanation}, the value of $\ell$ can simply be derived from whether or not $\vec{x} \models \varphi$ (note that $\vec{x} \models \psi$ necessarily needs to hold in order to violate Eq.~\eqref{eq:explanation}):
if $\vec{x} \models \varphi$, then this means that any future conjecture $\varphi'$ has to satisfy $\vec{x} \not\models \varphi'$ (otherwise $\vec{x}$ would again witness a violation of Eq.~\eqref{eq:explanation}) and, hence, $\ell = 0$;
conversely, if $\vec{x} \not\models \varphi$, then $\ell = 1$, indicating that every future conjecture needs to satisfy $\vec{x} \models \varphi'$.

\subsection{Synthesizer}
\label{sec:synthesizer}
The synthesizer's task is generate candidate FO formulas based on the concrete data (i.e., the counterexamples) it has received so far from the verifier.
To this end, we assume that the synthesizer maintains a finite set $\mathcal S \subset \mathbb R^n \times \{ 0, 1 \}$, called \emph{sample}, in which it stores the counterexamples.

In every iteration of the feedback loop, the synthesizer is asked to construct an FO formula $\varphi$ that is \emph{consistent with $\mathcal S$} in the sense that $\vec{x} \models \varphi$ for each $(\vec{x}, 1) \in \mathcal S$ and $\vec{x} \not\models \varphi$ for each $(\vec{x}, 0) \in \mathcal S$.
Note that we here look for a perfect classifier (which does not make any mistake), unlike what is usually the case in machine learning.

To solve this logic synthesis\slash learning problem, we resort to a framework called \emph{Syntax-Guided Synthesis (SyGuS)}~\cite{DBLP:conf/fmcad/AlurBJMRSSSTU13}.
Intuitively, SyGuS provides a standardized way to synthesize quantifier-free FO formulas from logical specifications, which are themselves expressed in first-order logic.
In this work, we do not require the full power of the SyGuS framework but only its ability to synthesize FO formulas from semantic constraints that are provided as input-output examples.
We use such input-output constraints to ensure that the resulting FO formula is consistent with the current sample.

A defining feature of the SyGuS framework is its ability to synthesize formulas that adhere to user-provided syntactic constraints (in addition to semantic constraints).
These syntactic constraints are given as a context free grammar, which defines the permissible solutions to the synthesis problem.
This has two advantages:
first, it allows the user to infuse domain knowledge into the synthesis process (e.g., the desired formula needs to be in negation normal form, it does or does not use certain Boolean operators, or specific constants do or do not occur);
second; it restricts the search space of possible solution and, hence, improves the performance of the synthesis process.
Note that this feature is optional: if no grammar is provided, the solution can be a quantifier-free FO formula in its most general form.

SyGuS is an active field of research, supporting numerous background theories (such as linear integer and real arithmetic, the theory of strings, the theory of bitvectors, etc.), and various mature synthesis engines exist, all of which understand the SyGuS interchange format.
Most of these engines are powered by symbolic enumeration techniques and\slash or powerful SMT solvers, such as Z3~\cite{DBLP:conf/tacas/MouraB08} or CVC4~\cite{DBLP:conf/cav/BarrettCDHJKRT11}.
Recently, this area has made significant progress, and modern SyGuS engines are often able to synthesize FO formulas\slash function from moderate-size real-world specifications.

Implementing the synthesizer is now straightforward.
We translate the sample into SyGuS-IF, provide a context free grammar (if desired by the user), and then invoke a SyGuS engine.
Once an FO formula has been synthesized, the verifier hands it over to the verifier as a new conjecture.
Note that it is easy to experiment with different SyGuS engines as they all support the SyGuS interchange format.


\subsection{Theoretical Analysis of the Framework}
\label{sec:theoretical-analysis}

We claim that if our framework returns a formula $\varphi^\star$, say after $m \geq 1$ iterations of the feedback loop, then $\varphi^\ast$ is an $\varepsilon$-explanation with probability at least $1 - \delta$.
To prove that this is in fact true, we observe that the probability of $\varphi^\star$ not being an $\varepsilon$-explanation (i.e., $\mathbb P_\mathcal D \bigl[ \vec{x} \in V_\mathcal M(\varphi, \psi, c) \bigr] \geq \varepsilon$) even if all test inputs have passed all of the checks of the verifier is at most
\begin{align*}
	\sum_{i = 1}^m (1 - \varepsilon)^{|T_i|} & \leq \sum_{i = 1}^m e^{-\varepsilon |T_i|} \\
	& = \sum_{i = 1}^m e^{-\varepsilon \left\lceil \frac{1}{\varepsilon} (i \cdot \ln{2} - \ln{\delta}) \right\rceil} \\
	& \leq \sum_{i = 1}^m e^{-\varepsilon \cdot \frac{1}{\varepsilon} (i \cdot \ln{2} - \ln{\delta})} = \sum_{i = 1}^m e^{-i \cdot \ln{2} + \ln{\delta}} \\
	& = \sum_{i = 1}^m e^{-i \cdot \ln{2}} \cdot e^{\ln{\delta}} = \sum_{i = 1}^m 2^{-i} \cdot \delta \\
	& \leq \delta.
\end{align*}
Thus, $\varphi^\star$ is indeed an $\varepsilon$-explanation with probability at least $1 - \delta$, which proves our main result.

\begin{theorem} \label{thm:main}
Let $\mathcal M \colon \mathbb R^n \to C$ be a model, $c \in C$ a class label, $\psi$ a query, $\mathcal D$ a probability distribution over $\mathbb R^n$, $\varepsilon \in (0, 1)$ an approximation parameter, and $\delta \in (0,1)$ confidence parameter.
If our framework terminates, it outputs a quantifier-fee FO formula $\varphi$ that is an $\varepsilon$-explanation for $\mathcal N$, $c$, and $\psi$ with probability at least $1 - \delta$.
\end{theorem}

In general, we cannot guarantee that our framework terminates.
On the one hand, the chosen syntactic fragment of FO (e.g., linear real arithmetic) might not expressive enough to capture the decision boundary of a model and, hence, the feedback loop continues forever, approximating the boundary better and better.
On the other hand, even if an $\varepsilon$-explanation in the the chosen fragment exists, the synthesizer might not find one (as the search space might be infinite), in which case the loop also continues forever.
However, following Löding, Madhusudan, and Neider~\cite{DBLP:conf/tacas/LodingMN16}, we identify two practical settings for which we can guarantee the termination of our framework:
\begin{enumerate}
	\item The syntactic constraints permit only a finite number of FO formulas (which can easily be enforced using an appropriate grammar). In the vocabulary of computational learning theory, this mean that the so-called \emph{concept class} of potential solutions is finite.
	\item There exists a total order $\prec$ on the considered fragment of FO (e.g., the lexicographic order over the string representations of formulas) and the SyGuS engine is able to construct \emph{$\prec$-minimal} formulas that satisfy the semantic and syntactic constraints.
\end{enumerate}
	
Let us begin with the first setting, where we assume that the syntactic constraints permit only a finite number of FO formulas.
A practical relevant example is the class of conjunctions over a large but finite fixed set of predicates (i.e., Boolean features).
For this specific setup, \framework\ either terminates and returns an $\varepsilon$-explanation, or it reports that none exists.
To simplify the following exposition, we assume that the SyGuS engine signals if none of the finitely many FO formula satisfies the syntactic and semantic constraints.\footnote{A trivial procedure to check this would be to enumerate all valid formulas in the syntactic fragment and check the semantic constraints (i.e., the input-output examples).}

\begin{theorem} \label{thm:finite-concept-class}
If the syntactic constraints of SyGuS permit only a finite number of formulas, then \framework\ is guaranteed to terminate.
On termination, it either returns an $\varepsilon$-explanation or reports that no explanation respecting the syntactic constraints exists.
\end{theorem}

For the remainder, let $\mathcal S_i$ be the sample in the $i$-th iteration of the feedback loop, where $i \in \{ 0, 1, \ldots \}$.
Moreover, let $\varphi_i$ be the explanation generated by \framework\ in the $i$-th iteration.

\begin{proof}[of Theorem~\ref{thm:finite-concept-class}]
We first make the following observation:
\begin{enumerate}
	\item \label{obs:explanations-semantically-distinct}
	The explanations generated during the run of \framework\ are semantically distinct (i.e., there exists $\vec{x} \in \mathbb R^n$ such that for all $i \neq j$, we have $\vec{x} \models \varphi_i$ if and only if $\vec{x} \not\models \varphi_j$).
	We first show this for two consecutive explanations $\varphi_i$ and $\varphi_{i+1}$: since \framework\ always constructs explanations that are consistent with the current sample and the counterexample $(\vec{x}_i, \ell_i)$ was added to $\mathcal S_i$ to form $\mathcal S_{i+1}$, we have $\vec{x}_i \not\models \varphi_i$ and $\vec{x}_i \models \varphi_{i+1}$.
	An analogous argument then shows that this is also true for the explanations $\varphi_i$ and $\varphi_j$ for each $j \in \{ 0, \ldots, i-1 \}$, using the counterexample $(\vec{x}_j, \ell_j)$ of Iteration~$j$ as a witness.
\end{enumerate}

If the syntactic restrictions of SyGuS only permit a finite number of distinct formulas, then \framework\ either eventually conjectures one that passes the verifier's test or it will have exhausted all possible FO formulas (in that syntactic fragment).
In the first case, \framework\ has found an $\varepsilon$-explanation and terminates.
In the second case, the verifier returns an additional counterexample.
However, \framework\ has already exhausted all semantically distinct formulas in the syntactic fragment.
Thus, the SyGuS engine aborts and reports that there is no FO formula in the syntactic fragment that satisfies the semantic constraints (i.e., that is consistent with the current sample).
Once this happens, \framework\ terminates as well and reports that there exist no explanation in the chosen syntactic fragment. \qed
\end{proof}

Let us now consider the second setting, where we assume that the SyGuS engine can always generate formulas that
\begin{enumerate*}[label={(\alph*)}]
	\item satisfy the syntactic and semantic constraints and
	\item are minimal with respect to a total order $\prec$ over the set of all syntactically valid FO formulas.
\end{enumerate*}
Löding, Madhusudan, and Neider~\cite{DBLP:conf/tacas/LodingMN16} call such synthesis engines \emph{Occam learners}, and we adopt this terminology here.
In fact, if a perfect explanation exists (which might not be unique) and \framework\ uses an Occam learner, then it is indeed guaranteed to find an $\varepsilon$-explanation in finite time (though not necessarily a perfect one).

\begin{theorem} \label{thm:occam-learner}
If a perfect explanation exists and \framework\ uses an Occam learner, then it is indeed guaranteed to terminate and return an $\varepsilon$-explanation.
\end{theorem}

\begin{proof}[of Theorem~\ref{thm:occam-learner}]
First, we make the following two observation:
\begin{enumerate}[start=2]
	\item \label{obs:samples-grow-strictly-monotonically}
	The sequence $\mathcal S_0, S_1, \ldots$ of samples generated in each iteration grows strictly monotonically (i.e., $\mathcal S_0 \subsetneq S_1 \subsetneq \cdots$).
	It is not hard to verify that the sequence grows monotonically since \framework\ always adds counterexamples to $\mathcal S$ but never removes them.
	The strictness arises from the fact that \framework\ always constructs hypotheses that are consistent with the current sample.
	Thus, the explanation $\varphi_i$ of Iteration~$i$ is consistent with $\mathcal S_i$, but the counterexample $(\vec{x}_i, \ell_i)$ was added (forming $\mathcal S_{i+1}$) because $\varphi_i$ was incorrect on this counterexample.
	Thus, $(\vec{x}_i, \ell_i)$ cannot have been an element of $\mathcal S_i$ and, hence, $\mathcal S_i \subsetneq \mathcal S_{i+1}$.
	\item \label{obs:formulas-grow-strictly-monotonically}
	We have $\varphi_i \prec \varphi_{i+1}$ for all $i$.
	Towards a contradiction, assume that $\varphi_i \succ \varphi_{i+1}$ (note that $\varphi_i = \varphi_{i+1}$ is not possible due to Observation~\ref{obs:explanations-semantically-distinct}).
	Since \framework\ always computes consistent formulas and $\mathcal S_i \subsetneq \mathcal S_{i+1}$ (see Observation~\ref{obs:samples-grow-strictly-monotonically}), we know that $\varphi_{i+1}$ is not only consistent with $\mathcal S_{i+1}$ but also with $\mathcal S_i$ (by definition of consistency).
	Moreover, the SyGuS engine is an Occam learner that always computes minimal consistent FO formulas.
	Hence, since $\varphi_{i+1}$ is consistent with $\mathcal S_i$ and $\varphi_i \succ \varphi_{i+1}$, the formula $\varphi_i$ cannot have been minimal with respect to $\prec$, which is a contradiction.
\end{enumerate}

Let $\varphi^\star$ now denote a perfect explanation.
Due to Observation~\ref{obs:formulas-grow-strictly-monotonically}, \framework\ will either terminate with an $\varepsilon$-explanation (which might or might not be perfect) or has exhausted all FO formulas $\varphi$ with $\varphi \prec \varphi^\star$.
Since all samples generated during the run of \framework\ are always consistent with $\varphi^\star$ (since the verifier always returns counterexamples from the set $V_\mathcal M(\varphi, \psi, c)$ and by virtue of the fact that all perfect explanation coincide inside the query $\psi$), the SyGuS engine necessarily generates $\varphi^\star$ in the next iteration.
Since $\varphi^\star$ is a perfect explanation, it is also an $\varepsilon$-explanation that passes the verifier's test.
Thus, \framework\ terminates and returns an $\varepsilon$-explanation.\qed
\end{proof}


\section{Experimental Evaluation}
\label{sec:evaluation}

In this section, we evaluate the explanations generated by {\framework}. We first discuss the experimental setup and the objective of the experiments and later discuss the experiment results. 

\subsection{Experimental Setup}
We have implemented a prototype of {\framework} in Python. The core technical component of the learner in {\framework} relies on solving appropriately designed SyGuS instances, and to this end, we have employed CVC4~\cite{DBLP:conf/cav/BarrettCDHJKRT11} as the solver for SyGuS. In SyGuS, we set the logic of the instance to LRA (Linear Real Arithmetic) to handle practical machine learning benchmarks consisting of real and categorical features. As discussed earlier, {\framework} allows to specify the format of the synthesized formula through the syntactic constraints of SyGuS. In our implementation, we have specified the syntactic constraints to learn formulas in DNF (Disjunctive Normal Form). For a real-valued feature, SyGuS can learn a real-valued constant with which the feature is compared such that the feature along with the compared constant is regarded as a Boolean predicate in the DNF formula. In our implementation, we have specified a finite set of constants for each real-valued feature. 

As a black-box, we have employed the neural-network library of the scikit-learn module in Python~\cite{scikit-learn}. Although {\framework} considers the neural network as a black-box, in our implementation, we first train a network on practical ML benchmarks and later explain it using DNF formulas. For the training of the neural network, the classifier is set to its default choices of the parameters as specified by the scikit-learn module.\footnote{As a future work, we explore different parameter choices and its effect on the generated explanations.} 
In our experiment, we have included three real-world benchmarks: Zoo, Iris, and Adult from the UCI repository~\cite{Dua:2019}.

The objectives of our empirical studies are as follows. 

\begin{itemize}
	\item Can {\framework} generate accurate and interpretable explanations? 
	\item What is the effect of the syntactic constraints in SyGuS on the generated explanations? 
\end{itemize}

We first discuss the experimental results on a synthetic benchmark and then extend discussions to practical benchmarks in the following. 

\subsection{Experimental Results}

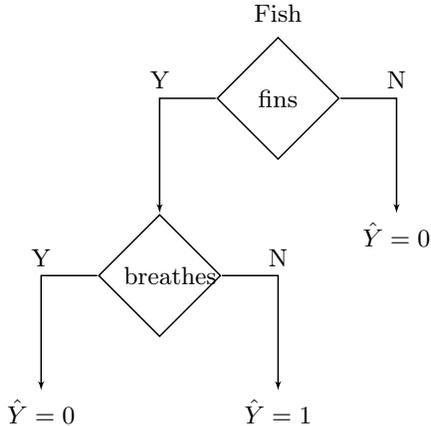
\begin{figure}
	\center
	\scalebox{.7}{	
		\begin{tikzpicture}[x=1cm,y=2cm]
		\node [box, scale=1.5]                                    (p)      {fins};
		\node [scale=1.5, above= 0.1cm of p]  (a)    {Fish};
		\node [scale=1.5, box, below= of p, xshift=-1.5cm]    (a1)    {breathes};
		\node [scale=1.5,below= of a1, xshift=-1.5cm]  (a11)    { $\hat{Y}= 0$};
		\node [scale=1.5,below= of a1, xshift=1.5cm]   (a12)    { $\hat{Y}=1 $};
		\node [scale=1.5,below= of p, xshift=1.5cm]  (a22)    { $\hat{Y}= 0$};
		\path [line] (p) -|         (a1) node [scale=1.5,midway, above]  {Y};
		\path [line] (p) -|         (a22) node [scale=1.5,midway, above]  {N};
		\path [line] (a1) -|       (a11) node [scale=1.5,midway, above]  {Y};
		\path [line] (a1) -|       (a12) node [scale=1.5,midway, above]  {N};
		\end{tikzpicture}}
		\caption{A decision tree classifier for predicting `fish' in the Zoo animal dataset. }
		\label{fig:sanity_decision_tree}
	
\end{figure}

\subsubsection{Synthetic Benchmark.}

We have tested {\framework} on `Zoo animal classification dataset' in order to verify whether {\framework} generates correct explanations. In this experiment, we have considered a decision tree classifier as a black-box (Figure~\ref{fig:sanity_decision_tree}), where the expected explanation can be understood well and it is easy to verify whether or not an explanation is correctly describing the decision boundary . The classifier can predict whether an animal is `fish' or not depending on two Boolean features: fins and breathes among a total of 16 features. In Table~\ref{table:sanity}, we list a set of queries and the generated explanations for the decision tree in Figure~\ref{fig:sanity_decision_tree}. For this experiment, we have set $ \varepsilon = 0.05 $ and $ \delta = 0.05 $ in {\framework}. 

When the query is TRUE, we ask for an explanation on all inputs. In that case, {\framework} learns an explanation  `$ \neg $ breathes $ \wedge $ fins', that is the DNF representation of the decision tree in Figure~\ref{fig:sanity_decision_tree}. Intuitively, an input that satisfies the DNF formula `$ \neg $ breathes $ \wedge $ fins' is predicted positive by the black-box and vice versa. The size of this explanation is two as there are two literals in the generated explanation.  {\framework} takes around $ 0.42 $ second in learning and verifying this explanation where the learner takes $ 0.32 $ second and the verifier takes $ 0.1 $ second. Furthermore, the verifier tests on average $ 130 $ random inputs before certifying that the generated explanation is an $ \varepsilon $-explanation of the black-box with confidence $ 1 - \delta $.  Empirically, we have observed that this explanation has accuracy $ 1.0 $ on the test dataset. 

{\framework} allows to restrict the input domain by specifying the query. When we consider a query `$ \neg $ fins', we are interested in an explanation of the black-box  on inputs where the Boolean feature `fins  =  FALSE'. For this query, {\framework} learns an explanation `FALSE' meaning that if the feature `fins' is FALSE for an input, the input is predicted negative by the black-box. This reasoning can be validated by the decision tree in Figure~\ref{fig:sanity_decision_tree}, where the branch leading from `$ \neg $ fins' reaches the leaf node $ \hat{Y} = 0 $. 

When the query is `$ \neg $ breathes', the generated explanation is `fins'. This explanation is trivially verified by the decision tree in Figure~\ref{fig:sanity_decision_tree} where an input satisfying $ \neg $ breathes is predicted positive when fins = TRUE. Moreover, when the query is `breathes', {\framework} learns an explanation `FALSE' meaning that an input with `breathes = TRUE' is predicted `FALSE' by the classifier. Although such a prediction is not explicitly visible in the decision tree in Figure~\ref{fig:sanity_decision_tree}, {\framework} certifies this explanation as an $ \varepsilon $-explanation with confidence $ 1-\delta $. Empirically, this explanation has accuracy $ 1.0 $ on the test dataset.   Finally, we introduce another Boolean feature `milk' in the query and ask for an explanation of the black-box on inputs where milk = TRUE. {\framework} learns that such an input is predicted TRUE if the formula `$ \neg $ breathes $ \wedge $ fins' is satisfied. Intuitively, an animal with milk = TRUE is predicted as `fish' by the black-box classifier if for the same input, breathes = FALSE and fins = TRUE. 

\begin{table}
	\centering
	
	\caption{Example of explanations for the decision tree in Figure~\ref{fig:sanity_decision_tree}.}
	\label{table:sanity}
	\setlength{\tabcolsep}{2em}

	\begin{tabular}{llrrr}
		\toprule
		Query &                 Explanation & Size & Accuracy & Time \\
		\midrule
		TRUE &   $ \neg $ breathes $ \wedge $ fins &    2 &      1.0 & 0.42 \\
		$ \neg $ fins &                       FALSE &    1 &      1.0 & 0.06 \\
		$ \neg $ breathes &                        fins &    1 &      1.0 & 0.14 \\
		breathes &                       FALSE &    1 &      1.0 & 0.06 \\
		milk &   $ \neg $ breathes $ \wedge $ fins &    2 &      1.0 & 0.44 \\
		\bottomrule
	\end{tabular}

\end{table}

\subsubsection{Practical Benchmarks.}
We now discuss the experimental results on practical benchmarks where we consider a neural network as a black-box and aim to explain the working of the network around a specific input. In our implementation, we have considered a distance-based query where our goal is to learn explanations of the network in the vicinity of a specific input in the dataset. For that, we normalize each feature between $ [0,1] $ and set the maximum distance of the query  from the target input to $ 0.5 $. While there can be different choices of the distance function, we have  focused on the cosine distance function in this implementation. Additionally,  we have set $ \varepsilon = 0.05 $, $ \delta = 0.05 $, and timeout as $ 300 $ seconds in {\framework}. 
We present the average results over $ 300 $ iterations in Table~\ref{table:practical}.

 \begin{table}
 	\centering
 	
 	\caption{Result in practical benchmarks.}
 	\label{table:practical}
 	\setlength{\tabcolsep}{1em}

\begin{tabular}{lrrrrrr}
	\toprule
	Dataset & Explanation size & Accuracy & Time & Learner(\%) & Verifier(\%) & Test inputs \\
	\midrule
	Adult &              6.2 &     0.69 &  300 &       0.96 &        0.04 &         173 \\
	Iris &              7.3 &     0.85 &  300 &       0.98 &        0.02 &         105 \\
	Zoo &              6.5 &     0.89 &  300 &       0.99 &        0.01 &         220 \\
	\bottomrule
\end{tabular}

 \end{table}

In Table~\ref{table:practical}, we observe that {\framework} generates small-size explanations  containing on average less than eight Boolean predicates in all three datasets. Therefore, the generated explanations are highly interpretable (see Table~\ref{table:explanations}). While the explanations are  succinct, the accuracy of the explanation is between $ 0.7 $ to $ 0.9 $.  The reason behind the poor accuracy of the explanations can be attributed by the fact that {\framework} times out in all three datasets and cannot certify $ (\epsilon, \delta) $ guarantee of the explanations. In Table~\ref{table:practical}, the learner takes majority of the allotted time than the verifier. The high execution time of the learner mostly depends on the syntactic constraints of SuGuS in order to learn interpretable explanations of a certain format, which in our case is DNF. We next discuss the effect of this syntactic constraints on the generated explanations.

\begin{table}
	\centering
	
	\caption{Examples of generated explanations in the vicinity of a specific input in the dataset. Each explanation for the black-box neural network is a first order formula in DNF that if satisfied, predicts positive class. }
	\label{table:explanations}
	\setlength{\tabcolsep}{1em}

	\begin{tabular}{ll}
		\toprule
		Dataset & \multicolumn{1}{c}{Explanation} \\
		\midrule
		Adult &             ((age $ < 0.25 $) $ \mathsf{AND} $ (education-num $ > 0.5 $)) $ \mathsf{OR} $  \\ &
		((education-num $ < 0.5 $) $ \mathsf{AND} $ ( capital-gain $ < 0.5 $) $ \mathsf{AND} $( hours-per-week $ <  0.75 $))  \\
		\\
		Iris &              ((sepal-width $ < 0.5 $) $ \mathsf{AND} $ (petal-length $ > 0.75 $))
		$ \mathsf{OR} $ \\ &
		((sepal-length $ > 0.25 $) $ \mathsf{AND} $  (sepal-width $ < 0.5 $) $ \mathsf{AND} $ (petal-length $ < 0.5 $) $ \mathsf{AND} $ (petal-width $ < 0.75 $))
		\\
		\\
		Zoo &              ($ \neg $ feathers $ \mathsf{AND} $ milk) $ \mathsf{OR} $  (backbone $ \mathsf{AND} $  $ \neg $ fins $ \mathsf{AND} $ catsize)  \\
		
		\bottomrule
	\end{tabular}

\end{table}

\begin{figure}[!t]
	\centering
	\includegraphics[scale=.4]{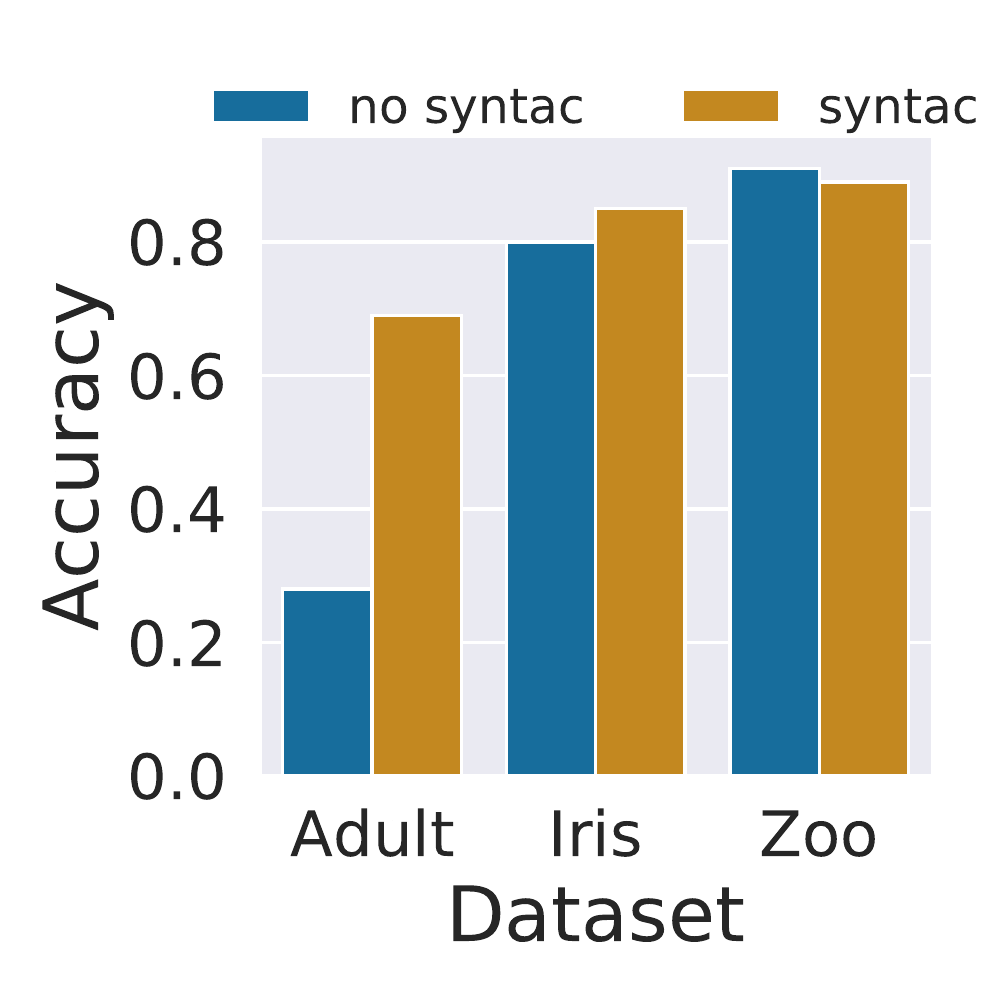}\hfill
	\includegraphics[scale=.4]{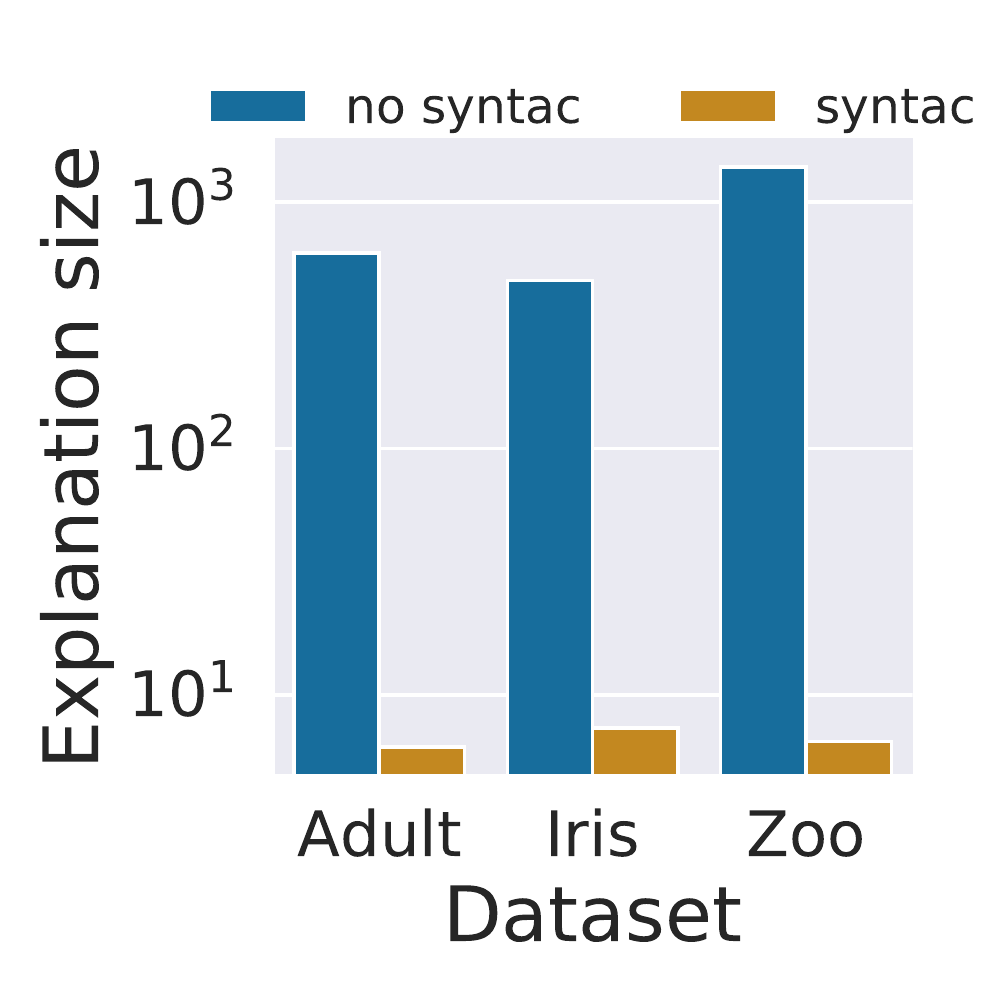}\hfill
	\includegraphics[scale=.4]{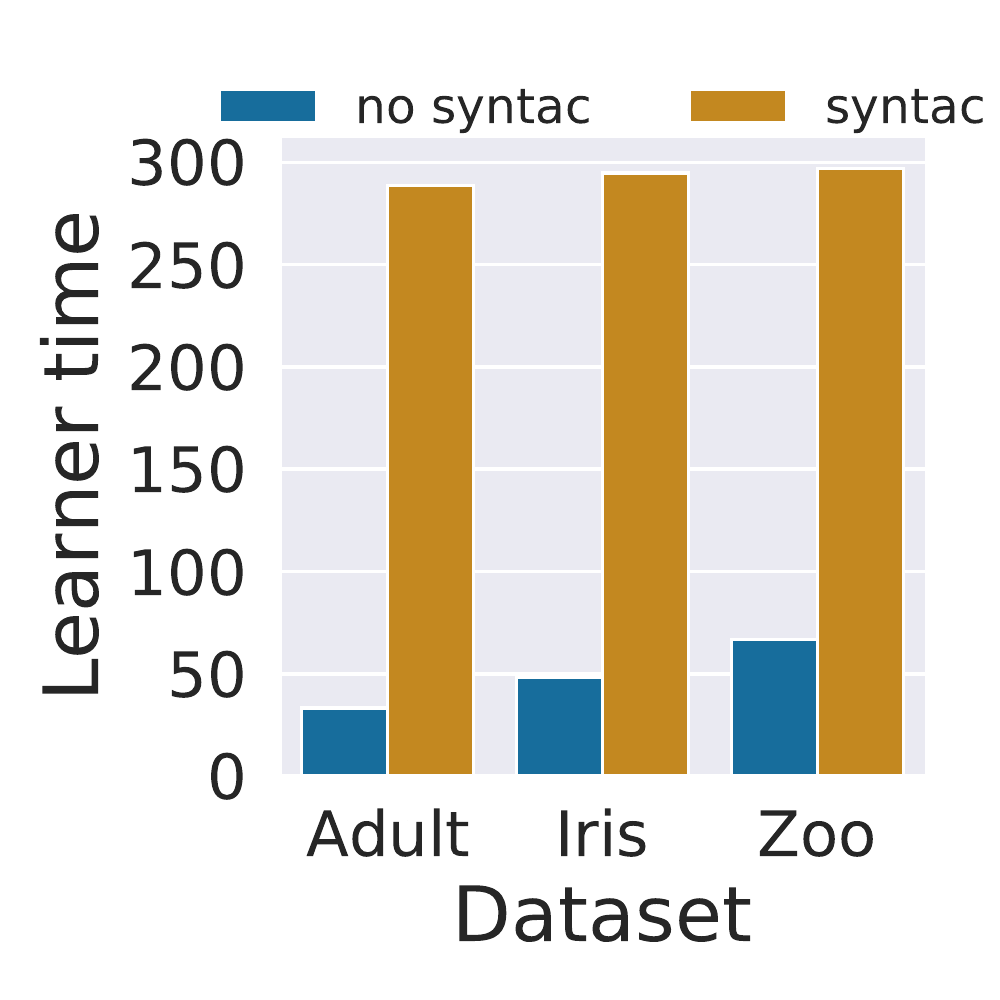}\hfill
	\includegraphics[scale=.4]{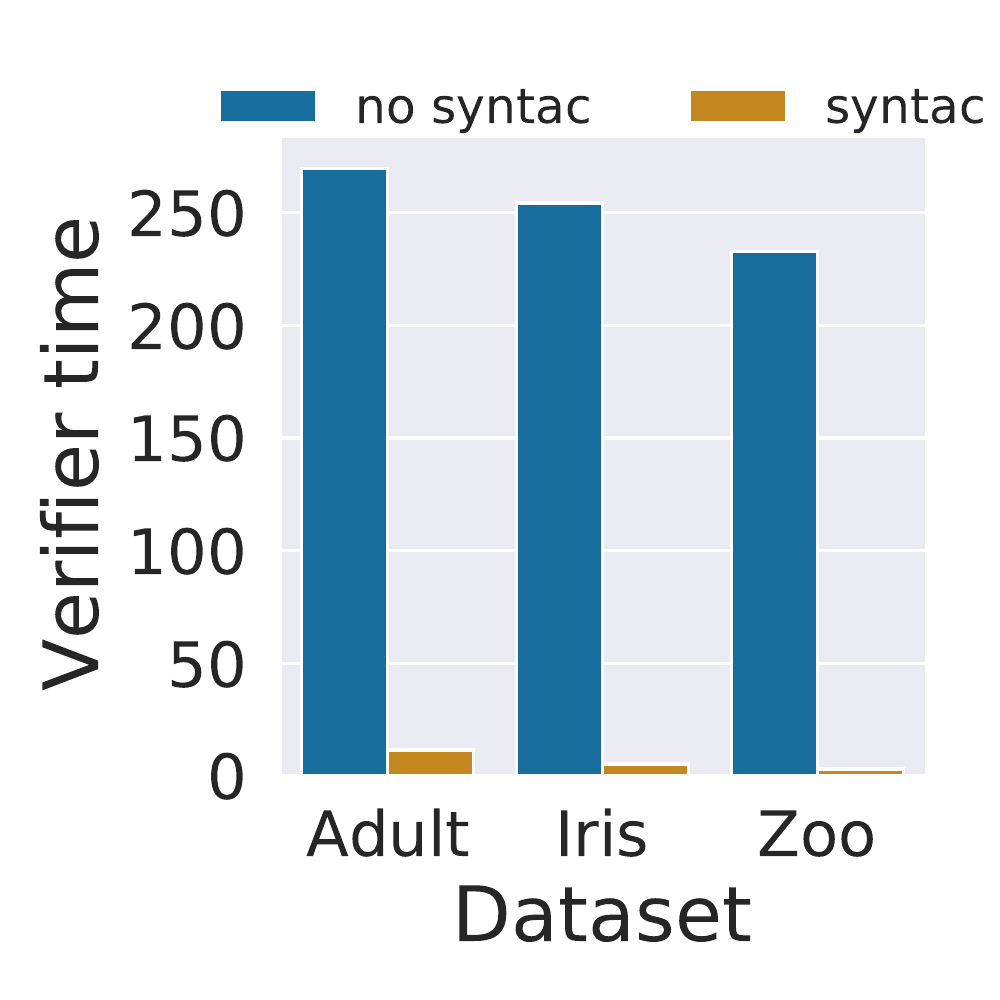}\hfill
	
	\caption{Effect of syntactic constraints (for learning DNF explanations) in SuGuS on the generated explanations. In all figures, `syntac' refers to including syntactic constraints in SyGuS. In the second plot, y-axis is in log scale.}
	\label{fig:syntactic_constraints}
\end{figure}

In Figure~\ref{fig:syntactic_constraints}, we present the accuracy and size of explanations along with execution time  in order to understand the effect of syntactic constraints in SyGuS for synthesizing explanations in DNF. In the leftmost plot, we observe that the accuracy of the explanation without any syntactic constraint is almost equal to applying syntactic constraints in the Zoo dataset but worse in both  Adult and Iris dataset. Thus the syntactic constraints for learning DNF not only restrict the format of the explanation, but also assist in generalization of the explanation.
Furthermore, major differences are observed in terms of size of the explanation and execution time. Particularly, the explanation are much larger in size (second plot) when no constraint is specified whereas the execution time of the learner is significantly less (third plot) because SyGuS  can quickly learn an explanation in the most general form when the syntactic constraints are not specified. Finally, we observe that although the execution time of the learner is much less without any syntactic constraint, the verifier needs to spend most of the allotted time (rightmost plot) because explanations in the most general form are too specific to the given set of counterexamples and cannot generalize well.

In summary, we observe that {\framework} can generate human interpretable succinct explanation of the neural network. The accuracy of explanations and the performance of {\framework} mostly depend on the syntactic constraints in SyGuS, where the choice of the constraints define the trade-off among accuracy, explanation size and execution time.


\section{Conclusion and Future Work}
\label{sec:conclusion}

In this paper, we have presented {\framework}, a novel explanation framework for generating probably approximately correct (PAC) explanations of the working of the black-box machine learning (ML) models. {\framework} is a model-agnostic and local explainer that is built on a novel integration of Syntax-Guided Synthesis (SyGuS) with PAC learning.  In the empirical studies, {\framework} generates small-size explanations  in first-order logic formulas that are highly interpretable.

For future work, we extend experimental evaluations to  state-of-the-art local explainers of the black-box ML models. Additionally, we explore different representations of explanations by exploiting the syntactic constraints in SyGuS.

\bibliographystyle{splncs04}
\bibliography{main}

\begin{thebibliography}{10}
\providecommand{\url}[1]{\texttt{#1}}
\providecommand{\urlprefix}{URL }
\providecommand{\doi}[1]{https://doi.org/#1}

\bibitem{DBLP:conf/fmcad/AlurBJMRSSSTU13}
Alur, R., Bod{\'{\i}}k, R., Juniwal, G., Martin, M.M.K., Raghothaman, M.,
  Seshia, S.A., Singh, R., Solar{-}Lezama, A., Torlak, E., Udupa, A.:
  Syntax-guided synthesis. In: Formal Methods in Computer-Aided Design, {FMCAD}
  2013, Portland, OR, USA, October 20-23, 2013. pp.~1--8. {IEEE} (2013),
  \url{http://ieeexplore.ieee.org/document/6679385/}

\bibitem{DBLP:conf/cav/BarrettCDHJKRT11}
Barrett, C.W., Conway, C.L., Deters, M., Hadarean, L., Jovanovic, D., King, T.,
  Reynolds, A., Tinelli, C.: {CVC4}. In: Gopalakrishnan, G., Qadeer, S. (eds.)
  Computer Aided Verification - 23rd International Conference, {CAV} 2011,
  Snowbird, UT, USA, July 14-20, 2011. Proceedings. Lecture Notes in Computer
  Science, vol.~6806, pp. 171--177. Springer (2011).
  \doi{10.1007/978-3-642-22110-1\_14}

\bibitem{DBLP:journals/tc/BiermannF72}
Biermann, A.W., Feldman, J.A.: On the synthesis of finite-state machines from
  samples of their behavior. {IEEE} Trans. Computers  \textbf{21}(6),  592--597
  (1972). \doi{10.1109/TC.1972.5009015}

\bibitem{Dua:2019}
Dua, D., Graff, C.: {UCI} machine learning repository (2017),
  \url{http://archive.ics.uci.edu/ml}

\bibitem{DBLP:conf/atva/Ehlers17}
Ehlers, R.: Formal verification of piece-wise linear feed-forward neural
  networks. In: D'Souza, D., Kumar, K.N. (eds.) Automated Technology for
  Verification and Analysis - 15th International Symposium, {ATVA} 2017, Pune,
  India, October 3-6, 2017, Proceedings. Lecture Notes in Computer Science,
  vol. 10482, pp. 269--286. Springer (2017).
  \doi{10.1007/978-3-319-68167-2\_19}

\bibitem{DBLP:conf/cav/ElboherGK20}
Elboher, Y.Y., Gottschlich, J., Katz, G.: An abstraction-based framework for
  neural network verification. In: Lahiri, S.K., Wang, C. (eds.) Computer Aided
  Verification - 32nd International Conference, {CAV} 2020, Los Angeles, CA,
  USA, July 21-24, 2020, Proceedings, Part {I}. Lecture Notes in Computer
  Science, vol. 12224, pp. 43--65. Springer (2020).
  \doi{10.1007/978-3-030-53288-8\_3}

\bibitem{goodman2017european}
Goodman, B., Flaxman, S.: European union regulations on algorithmic
  decision-making and a “right to explanation”. AI magazine
  \textbf{38}(3),  50--57 (2017)

\bibitem{DBLP:conf/ijcai/Ignatiev20}
Ignatiev, A.: Towards trustable explainable {AI}. In: Bessiere, C. (ed.)
  Proceedings of the Twenty-Ninth International Joint Conference on Artificial
  Intelligence, {IJCAI} 2020. pp. 5154--5158. ijcai.org (2020).
  \doi{10.24963/ijcai.2020/726}

\bibitem{DBLP:conf/aaai/IgnatievNM19}
Ignatiev, A., Narodytska, N., Marques{-}Silva, J.: Abduction-based explanations
  for machine learning models. In: The Thirty-Third {AAAI} Conference on
  Artificial Intelligence, {AAAI} 2019, The Thirty-First Innovative
  Applications of Artificial Intelligence Conference, {IAAI} 2019, The Ninth
  {AAAI} Symposium on Educational Advances in Artificial Intelligence, {EAAI}
  2019, Honolulu, Hawaii, USA, January 27 - February 1, 2019. pp. 1511--1519.
  {AAAI} Press (2019). \doi{10.1609/aaai.v33i01.33011511}

\bibitem{DBLP:conf/cav/KatzBDJK17}
Katz, G., Barrett, C.W., Dill, D.L., Julian, K., Kochenderfer, M.J.: Reluplex:
  An efficient {SMT} solver for verifying deep neural networks. In: Majumdar,
  R., Kuncak, V. (eds.) Computer Aided Verification - 29th International
  Conference, {CAV} 2017, Heidelberg, Germany, July 24-28, 2017, Proceedings,
  Part {I}. Lecture Notes in Computer Science, vol. 10426, pp. 97--117.
  Springer (2017). \doi{10.1007/978-3-319-63387-9\_5}

\bibitem{DBLP:conf/tacas/LodingMN16}
L{\"{o}}ding, C., Madhusudan, P., Neider, D.: Abstract learning frameworks for
  synthesis. In: Chechik, M., Raskin, J. (eds.) Tools and Algorithms for the
  Construction and Analysis of Systems - 22nd International Conference, {TACAS}
  2016, Held as Part of the European Joint Conferences on Theory and Practice
  of Software, {ETAPS} 2016, Eindhoven, The Netherlands, April 2-8, 2016,
  Proceedings. Lecture Notes in Computer Science, vol.~9636, pp. 167--185.
  Springer (2016). \doi{10.1007/978-3-662-49674-9\_10}

\bibitem{DBLP:conf/tacas/MouraB08}
de~Moura, L.M., Bj{\o}rner, N.: {Z3:} an efficient {SMT} solver. In:
  Ramakrishnan, C.R., Rehof, J. (eds.) Tools and Algorithms for the
  Construction and Analysis of Systems, 14th International Conference, {TACAS}
  2008, Held as Part of the Joint European Conferences on Theory and Practice
  of Software, {ETAPS} 2008, Budapest, Hungary, March 29-April 6, 2008.
  Proceedings. Lecture Notes in Computer Science, vol.~4963, pp. 337--340.
  Springer (2008). \doi{10.1007/978-3-540-78800-3\_24}

\bibitem{scikit-learn}
Pedregosa, F., Varoquaux, G., Gramfort, A., Michel, V., Thirion, B., Grisel,
  O., Blondel, M., Prettenhofer, P., Weiss, R., Dubourg, V., Vanderplas, J.,
  Passos, A., Cournapeau, D., Brucher, M., Perrot, M., Duchesnay, E.:
  Scikit-learn: Machine learning in {P}ython. Journal of Machine Learning
  Research  \textbf{12},  2825--2830 (2011)

\bibitem{CEGIS}
Solar-Lezama, A.: Program synthesis by sketching. Ph.D. thesis, University of
  California at Berkeley (2008)

\bibitem{DBLP:journals/cacm/Valiant84}
Valiant, L.G.: A theory of the learnable. Commun. {ACM}  \textbf{27}(11),
  1134--1142 (1984). \doi{10.1145/1968.1972}

\end{thebibliography}

\end{document}